%% file: iclr2023_conference.tex
\documentclass{article} 
\usepackage{iclr2023_conference,times}

\input{math_commands.tex}

\usepackage{hyperref}
\usepackage{url}

\title{SeKron: A Decomposition Method Supporting Many Factorization Structures}

\author{Marawan Gamal Abdel Hameed,
    Ali Mosleh, Marzieh S. Tahaei,
    Vahid Partovi Nia\\
Noah’s Ark Lab, Huawei Technologies Canada \\ \\
\texttt{marawan.abdel.hameed@huawei.com}  \quad
\texttt{ali.mosleh@huawei.com}\\
\texttt{marzieh.tahaei@huawei.com}  \quad \texttt{vahid.partovinia@huawei.com}
}

\iclrfinalcopy 

\let\cite\citep
\setcitestyle{aysep={}}
\usepackage[utf8]{inputenc} 
\usepackage[T1]{fontenc}    
\usepackage{hyperref}       
\usepackage{url}            
\usepackage{booktabs}       
\usepackage{amsfonts}       
\usepackage{nicefrac}       
\usepackage{microtype}      

\usepackage{booktabs}
\usepackage{thmtools,thm-restate}
\input{marawan_math_commands}
\usepackage{multirow}
\usepackage{caption}
\usepackage{subcaption}
\usepackage{graphicx}
\usepackage{subcaption}

\usepackage{pgfplots}
\pgfplotsset{width=10cm,compat=1.9}

\usepackage{tikz}
\usetikzlibrary{positioning, arrows.meta, quotes}
\usetikzlibrary{shapes,snakes}
\usetikzlibrary{bayesnet}
\tikzset{>=latex}
\tikzstyle{plate caption} = [caption, node distance=0, inner sep=0pt,
below left=5pt and 0pt of #1.south]
\makeatletter
\newcommand{\ssymbol}[1]{^{\@fnsymbol{#1}}}
\makeatother
\usetikzlibrary{calc}

\usepackage[normalem]{ulem}

\renewcommand{\headrulewidth}{0pt}
\renewcommand{\footrulewidth}{0pt}

\begin{document}

\maketitle

\begin{abstract}
While convolutional neural networks (CNNs) have become the de facto standard for most image processing and computer vision applications, their deployment on edge devices remains challenging. Tensor decomposition methods provide a means of compressing CNNs to meet the wide range of device constraints by imposing certain factorization structures on their convolution tensors. However, being limited to the small set of factorization structures presented by state-of-the-art decomposition approaches can lead to sub-optimal performance. We propose SeKron, a novel tensor decomposition method that offers a wide variety of factorization structures, using sequences of Kronecker products. By recursively finding approximating Kronecker factors,
we arrive at optimal decompositions for each of the factorization structures. 
We show that SeKron is a flexible decomposition that generalizes widely used methods, such as Tensor-Train (TT), Tensor-Ring (TR),  Canonical Polyadic (CP) and Tucker decompositions. 
Crucially, we derive an efficient convolution projection algorithm shared by all SeKron structures, leading to seamless compression of CNN models.
We validate SeKron for model compression on both high-level and low-level computer vision tasks and find that it outperforms state-of-the-art decomposition methods. 
\end{abstract}

\vspace{-2mm}\section{Introduction}
\input{intro}

\vspace{-2mm}\section{Related Work on DNN Model Compression}
\input{related_work}

\vspace{-0mm}\section{Method} 

In this section, we introduce SeKron and how it can be used to compress tensors in deep learning models. We start by providing background on the Kronecker Product Decomposition in Section \ref{ss:prelim}. Then, we introduce our decomposition method in \ref{ss:kronecker-sequence-decomposition}. In Section \ref{ss:convolution-with-kronecker-sequence-decomposition}, we provide an algorithm for computing the convolution operation using each of the factors directly (avoiding reconstruction) at runtime. Finally, we discuss the computational complexity of the proposed method in Section \ref{ss:computational-complexity}.

\vspace{-0mm}
{
\setlength{\abovedisplayskip}{3pt}
\setlength{\belowdisplayskip}{3pt}
\subsection{Preliminaries}\label{ss:prelim}
Convolutional layers prevalent in CNNs transform an input tensor
$\tensor{x} \in \R^{C \times K_h \times K_w}$ using a weight tensor 
$\tensor{w} \in \R^{F \times C \times K_h \times K_w}$ via a multi-linear map given by 
\begin{equation}
  \tensor{y}_{f,x,y} = \sum_{i=1}^{K_h} \sum_{j=1}^{K_w} \sum_{c=1}^{C} \tensor{W}_{f,c,i,j} \tensor{X}_{c,i+x,j+y},
  \label{eq:conv2d-scalar-form}
\end{equation}
where $C$ and $F$ denote the number of input channels and output channels, respectively, and $K_h \times K_w$ denotes the spatial size of the weight (filter). Tensor decomposition seeks a compact approximation to replace $\tensor{w}$, typically through finding lower-rank tensors using SVD. 

One possible way of obtaining such a compact approximation comes from the fact that any 
tensor $\tensor{w} \in \R^{w_1 \times \cdots \times w_N}$ can be written as sum of Kronecker products \cite{gamal2022convolutional}. Let
the Kronecker product between two matrices $\mat{a} \in \R^{a_1 \times a_2}$ and $\mat{b} \in \R^{b_1 \times b_2}$ be given by 
\begin{equation}
    \mat{a} \otimes \mat{b} \triangleq
    \begin{bmatrix}
    A_{1 1} \mat{b}  &  \dots  &  A_{1a_2} \mat{b} \\
    \vdots          &  \ddots & \vdots \\
    A_{a_1 1} \mat{b}  &  \dots  &  A_{a_1 a_2} \mat{b} 
    \end{bmatrix},  
    \label{eq:def-kron-matrix-pairs}
\end{equation}
which can be extended to two multi-way tensors $\tensor{a} \in \R^{a_1  \times \dots \times a_N}$ and $\tensor{b} \in \R^{b_1\times \dots \times b_N}$ as
\begin{equation}
    (\tensor{A} \otimes \tensor{B})_{i_1  \cdots i_N} \triangleq \tensor{A}_{j_1 \cdots j_N} \tensor{B}_{k_1 \cdots k_N},
        \label{eq:def-kron-tensor-pairs}
\end{equation}
where
$
    j_n = \left \lfloor \frac{i_n}{b_n} \right\rfloor, \hspace{7pt}
    k_n = i_n\mod{b_n}
$.
Since $\tensor{w}$ can be written as a sum of Kronecker products, i.e.
$
    \tensor{w} = \sum_{r=1}^{R} \tensor{a}_r \otimes  \tensor{b}_r,
$
it can be approximated using a lower-rank representation by solving 
\begin{equation}
    \underset{\{\tensor{a}_r\}, \{\tensor{b}_r\}}{\min} \left\| \tensor{w} - \sum_{r=1}^{\widehat{R}} \tensor{a}_r \otimes  \tensor{b}_r \right\|_{\mathrm{F}}^2,
\label{eq:opt-nearest-kronecker-product-tensor-form}
\end{equation}
for $\widehat{R}$ sums of Kronecker products ($\widehat{R} \leq R$) using the SVD of a particular reshaping (unfolding) of $\tensor{w}$, where $||\cdot||_{\mathrm{F}}$ denotes the Frobenius norm. Thus, the convolutional layer can be compressed by replacing $\tensor{w}$ with $\sum_{r=1}^{\widehat{R}} \tensor{a}_r \otimes  \tensor{b}_r$ in \eqref{eq:conv2d-scalar-form}, where $\widehat{R}$ controls the compression rate. 
}

\subsection{SeKron Tensor Decomposition}\label{ss:kronecker-sequence-decomposition}

The multi-way tensor representation in \eqref{eq:def-kron-tensor-pairs} can be extended to sequences of multi-way tensors
\begin{equation}
\left(
\tensor{a}^{(1)} \otimes\cdots\otimes
\tensor{a}^{(S)} 
\right)_{i_1 \cdots i_N}
\triangleq
\tensor{a}_{j_1^{(1)} \cdots j_N^{(1)}}^{(1)} \cdots
\tensor{a}_{j_1^{(S)} \cdots j_N^{(S)}}^{(S)},
\label{eq:def-kron-tensor-sequences}
\end{equation}
where 
\begin{equation}
j_n^{(k)} =
    \begin{cases}
        i_n - \sum_{t=1}^{k-2} j_{n}^{(t)} \prod_{l=t+1}^{S} a_n^{(l)} \text{mod} \, a_n^{(S)} \quad & k=S, \\
    \left\lfloor \frac
    { i_n - \sum_{t=1}^{k-1} j_{n}^{(t)} \prod_{l=t+1}^{S} a_n^{(l)} }
    {\prod_{l=k+1}^{S} a_n^{(l)}}\right\rfloor \quad 
    & \text{otherwise},
    \end{cases}
    \label{eq:kron-tensor-sequences-indexing}
\end{equation}
and $\tensor{a}^{(k)} \in \R^{a_1^{(k)} \times \dots \times a_N^{(k)}}$.
Therefore, our decomposition method using sequences of Kronecker products i.e., SeKron, applied on a given tensor $\tensor{w} \in \R^{w_1 \times \cdots \times w_N}$ involves finding a solution to
\begin{equation}
    \underset{\{\tensor{a}^{(k)}\}_{k=1}^{S}}{\min}
    \left\| \tensor{w}  - 
    \sum_{r_{1}=1}^{R_{1}} \tensor{a}_{r_{1}}^{(1)} \otimes
    \sum_{r_{2}=1}^{R_{2}} \tensor{a}_{r_{1} r_{2}}^{(2)} \otimes\cdots\otimes
    \sum_{r_{S-1}=1}^{R_{S-1}} \tensor{a}_{r_{1} \cdots r_{S-1} }^{(S-1)}
    \otimes
    \tensor{a}_{r_{1} \cdots r_{S-1} }^{(S)}  \right\|_{\mathrm{F}}^{2}.
    \label{eq:opt-nearest-kronecker-sequence-tensor-form}
\end{equation}
Although this is a non-convex optimization problem, we provide a solution based on recursive application of SVD and demonstrate its quasi-optimality:

\begin{restatable}[Tensor Decomposition using a Sequence of Kronecker Products]{theorem}{ksdecomposition}
\label{thm:ksdecomposition}
Any tensor $\tensor{w}$ $\in$ $\R^{w_1  \times \cdots \times w_N}$ can be represented by a sequence of Kronecker products between $S\in \N$ factors:
\begin{equation}
    \tensor{w} =  \sum_{r_{1}=1}^{R_{1}} \tensor{a}_{r_{1}}^{(1)} \otimes
    \sum_{r_{2}=1}^{R_{2}} \tensor{a}_{r_{1} r_{2}}^{(2)} \otimes\cdots\otimes
    \sum_{r_{S-1}=1}^{R_{S-1}} \tensor{a}_{r_{1} \cdots r_{S-1} }^{(S-1)}
    \otimes
    \tensor{a}_{r_{1} \cdots r_{S-1} }^{(S)},
    \label{eq:ksd-full-rank-tensor-form}
\end{equation}
where $R_i$ are ranks of intermediate matrices and $\tensor{a}^{(k)} \in \R^{a_1^{(k)} \times \cdots \times a_N^{(k)}}$.
\end{restatable}
\begin{proof}
See Appendix
\end{proof}

{\setlength{\abovedisplayskip}{1pt}
\setlength{\belowdisplayskip}{1pt}
Our approach to solving \eqref{eq:opt-nearest-kronecker-sequence-tensor-form} involves finding two approximating Kronecker factors that minimize the reconstruction error with respect to the original tensor, then recursively applying this procedure on the latter factor found. More precisely, we define intermediate tensors
\begin{multline}
    \tensor{b}_{r_1 \cdots r_k}^{(k)} \triangleq 
    \sum_{r_{k+1}=1}^{R_{k+1}} \tensor{a}_{r_1 \cdots  r_{k+1}}^{(k+1)} \otimes
    \sum_{r_{k+2}=1}^{R_{k+2}} \tensor{a}_{r_1 \cdots r_{k+2}}^{(k+2)} \otimes
    \cdots
    \otimes
    \sum_{r_{S-1}=1}^{R_{S-1}} \tensor{a}_{r_1 \cdots r_{S-1}}^{(S-1)}
    \otimes
    \tensor{a}_{r_1 \cdots r_{S-1}}^{(S)},
    \label{eq:intermediate-tensors-full-rank}
\end{multline}
allowing us to re-write the reconstruction error in \eqref{eq:opt-nearest-kronecker-sequence-tensor-form}, for the $k^{th}$ iteration, as
\begin{equation}
    \underset{\substack{\{\tensor{a}_{r_{1} \cdots r_k}^{(k)}, \; \tensor{b}_{r_{1} \cdots r_k}^{(k)} \} \\ r_j = 1, \cdots R_j, \; j=1, \dots, k}}{\min}
    \left\| \tensor{w}_{r_1 \cdots r_{k-1}}^{(k)}  - 
    \sum_{r_{k}=1}^{R_{k}} \tensor{a}_{r_{1} \cdots r_k}^{(k)} \otimes
    \tensor{b}^{(k)}_{r_{1} \cdots r_k}
    \right\|_{\mathrm{F}}^{2}.
    \label{eq:opt-ksd-intermediate-tensors}
\end{equation}
In the first iteration, the tensor being decomposed is the original tensor (i.e.,  $\tensor{w}^{(1)} \gets \tensor{w}$). Whereas in  subsequent iterations, intermediate tensors are decomposed.
At each iteration, we can convert the problem in \eqref{eq:opt-ksd-intermediate-tensors} to the low-rank matrix approximation problem
\begin{equation}
    \underset{\substack{
    \{\vec{a}_{r_{1} \cdots r_k}^{(k)}, \; \vec{b}_{r_{1} \cdots r_k}^{(k)} \} \\ r_j = 1, \cdots R_j, \; j=1, \dots, k}}{\min}
    \left\| \mat{w}_{r_1 \cdots r_{k-1}}^{(k)}  - 
    \sum_{r_{k}=1}^{R_{k}} \vec{a}_{r_1 \cdots r_k}^{(k)} 
    \vec{b}^{(k) \top}_{r_1 \cdots r_k}
    \right\|_{\mathrm{F}}^{2},
    \label{eq:opt-ksd-intermediate-vectors}
\end{equation}
through reshaping, such that the overall sum of squares is preserved between \eqref{eq:opt-ksd-intermediate-tensors} and \eqref{eq:opt-ksd-intermediate-vectors}.
The problem in \eqref{eq:opt-ksd-intermediate-vectors} can be readily solved, as it has a well known solution using SVD. The reshaping operations that facilitate this transformation are
\begin{equation}
\mat{w}_{r_1 \cdots r_{k-1}}^{(k)} = \textsc{Mat}(\textsc{Unfold}(\tensor{w}_{r_1 \cdots r_{k-1}}^{(k)}, \vec{d}_{\tensor{b}_{r_1 \cdots r_k}^{(k)}})),
\label{eq:reshaping-operations-w}
\end{equation}
\vspace{-1mm}
\begin{equation}
\vec{a}_{r_1  \cdots r_k}^{(k)} =  \textsc{Unfold}(\tensor{a}_{r_1 \cdots r_{k}}^{(k)}, \vec{d}_{\tensor{i}_{\tensor{a}_{r_1 \cdots r_k}^{(k)}}}), \; \;
\vec{b}_{r_1 \cdots r_{k}}^{(k)} = \textsc{Vec}(\tensor{b}_{r_1 \cdots r_k}^{(k)}),
\label{eq:reshaping-operations-a-b}
\vspace{-2mm}
\end{equation}
where $\textsc{Unfold}$ reshapes tensor $\tensor{w}_{r_1 \cdots r_{k-1}}^{(k)}$ by extracting multidimensional patches of shape $\vec{d}_{\tensor{b}_{r_1, \dots r_k}^{(k)}}$ from tensor $\tensor{w}_{r_1 \cdots r_{k-1}}^{(k)}$ in any order,  $\vec{d}_{\tensor{b}}$ denotes a vector describing the shape of a tensor $\tensor{b}$, $\textsc{Vec}: \R^{d_{1} \times  \cdots \times d_{N}} \to \R^{d_{1} \cdots d_{N}}$ flattens a tensor to a 1D vector, $\textsc{Mat}: \R^{d_{1} \times \cdots \times d_{N}} \to \R^{d_{1} \cdots  d_{N}}$ reshapes a tensor to a matrix and  $\tensor{i}_{\tensor{a}}$ denotes an identity tensor which has the same number of dimensions as tensor $\tensor{a}$ with each dimension set to one.}

{\setlength{\textfloatsep}{0pt}
\setlength{\floatsep}{0pt}
\setlength{\intextsep}{0pt}
\setlength{\dbltextfloatsep}{0pt}
\begin{algorithm}[t]
    \KwIn{$ \text{Input tensor} \, \tensor{w} \in \R^{w_1 \times  \cdots \times w_N} \; \; \text{Kronecker factor shapes} \ \{d^{(i)}\}_{i=1}^{S}$}
    \KwOut{$\text{Kronecker factors} \; \{\tensor{a}\}_{i=1}^{S}$}

    \For{$i \gets 1, 2, \dots, S-1$}{
        $\vec{d}^{(a)} \gets \vec{d}^{(i)}$ \\
        $\vec{d}^{(b)} \gets \prod_{k=i+1}^{S} \vec{d}^{(k)}$ \\
        $\mat{w} \gets \textsc{Unfold}(\tensor{w}, \, \text{shape}=\vec{d}^{(b)})$ \tcp{$\R^{B \times L \times \prod_{k=1}^{N}d_{k}^{(b)}}$} 
        $\mat{u}, \vec{s}, \mat{v} \gets \textsc{SVD}(\mat{W})$ \tcp{$\mat{U} \in \R^{B \times L \times R}$ where $R=\text{min}(L, \prod_{k=1}^{N}d_{k}^{(b)})$}
        $\tensor{a}^{(i)} \gets \textsc{Stack}((\textsc{Reshape}(\mat{u}_{b, :, r}, \, \text{shape}=\vec{d^{(a)}}) \,|\, b=1, 2, \dots B, \; r=1, 2, \dots R))$ \\
        $\tensor{b}^{(i)} \gets \textsc{Stack}((\textsc{Reshape}(s_k \mat{v}^{\top}_{b, :,r}, \, \text{shape}=\vec{d^{(b)}}) \,|\, b=1, 2, \dots B, \; r=1, 2, \dots R))$\\
        $\tensor{w} \gets \tensor{b}^{(i)}$\\
    }
     $\tensor{a}^{(S)} \gets \tensor{B}^{S-1}$\\
    \Return $\{\tensor{a}\}_{i=1}^{S}$
    \caption{SeKron Tensor Decomposition}
    \label{alg:kronecker-sequence-decomposition}
\end{algorithm}}

Once each $\tensor{b}_{r_1 \cdots r_{k}}^{(k)}$ is obtained by solving \eqref{eq:opt-ksd-intermediate-vectors} (and using the inverse of the $\textsc{Vec}$ operation in \eqref{eq:reshaping-operations-a-b}), we proceed recursively by setting $\tensor{w}_{r_1 \cdots r_{k}}^{(k+1)} \gets \tensor{b}_{r_1 \cdots r_{k}}^{(k)}$ and solving the $k+1^{th}$ iteration of equation \eqref{eq:opt-ksd-intermediate-vectors}. In other words, at the $k^{th}$ iteration, we find Kronecker factors $\tensor{a}^{(k)}$ and  $\tensor{b}^{(k)}$, where the latter is used in the following iteration. Except in the final iteration (i.e., $k=S-1$), where the intermediate tensor $\tensor{b}^{(k)}$ is the solution to the last Kronecker factor $\tensor{a}^{(S)}$. Algorithm \ref{alg:kronecker-sequence-decomposition} presents this procedure, where for clarity, the recursive steps are unfolded in a for-loop.

By virtue of the connectivity between all of the Kronecker factors as illustrated in Figure~\ref{fig:tensor-networks-and-inefficiency-of-smaller-sequneces}a, SeKron generalizes many other decomposition methods. This result is formalized in the following theorem:

\begin{restatable}[Generality of SeKron]{theorem}{sekrongenerality}
\label{thm:SeKron-generality}
The CP, Tucker, TT and TR decompositions of a given tensor $\tensor{w} \in \R^{w_1 \times \cdots \times w_N}$ are special cases of its SeKron decomposition.
\end{restatable}
\begin{proof}
See Appendix.
\end{proof}

\subsection{Convolution with SeKron Structures}\label{ss:convolution-with-kronecker-sequence-decomposition}

{
\setlength{\abovedisplayskip}{0pt}
\setlength{\belowdisplayskip}{0pt}
In this section, we provide an efficient algorithm for performing a convolution operation using a tensor represented by a sequence of Kronecker factors. Assuming $\tensor{w}$ is approximated as a sequence of Kronecker products using SeKron, i.e., $\tensor{w} \approx \widehat{\tensor{w}}$ and
\begin{equation}
    \widehat{\tensor{w}} = 
    \sum_{r_1 = 1}^{\widehat{R}_1} \tensor{a}_{r_1}^{(1)} \otimes
    \sum_{r_2 = 1}^{\widehat{R}_2} \tensor{a}_{r_1 r_2}^{(2)} \otimes
    \dots
    \sum_{r_{S-1} = 1}^{\widehat{R}_{S-1}} \tensor{a}_{r_1 \cdots r_{S-1}}^{(S-1)} \otimes
    \tensor{a}_{r_1 \cdots r_{S-1}}^{(S)},
    \label{eq:ksd-low-rank-tensor-form}
\end{equation}
the convolution operation in \eqref{eq:conv2d-scalar-form} can be re-written as
\begin{multline}
  \tensor{y}_{fxy} = 
  \sum_{\substack{i,j,c=1}}^{K_h, K_w, C}
  \Bigg(
  \sum_{r_1 = 1}^{\widehat{R}_1} \tensor{a}_{r_1}^{(1)} 
  \otimes
    \cdots 
    \otimes
    \sum_{r_{S-1} = 1}^{\widehat{R}_{S-1}} \tensor{a}_{r_1 \cdots r_{S-1}}^{(S-1)} \otimes
    \tensor{a}_{r_1 \cdots r_{S-1}}^{(S)} 
    \Bigg)_{fcij} 
  \tensor{X}_{c,i+x,j+y}.
  \label{eq:kseqconv2d-reconstructed-scalar-form}
\end{multline}
Due to the factorization structure of tensor $\widehat{\tensor{w}}$, the projection in \eqref{eq:kseqconv2d-reconstructed-scalar-form} can be carried out without its explicit reconstruction. Instead, the projection can be performed using each of the Kronecker factors \emph{independently}. This property is essential to performing efficient convolution operations using SeKron factorizations, and leads to a reduction in both memory and FLOPs at runtime. In practice, this amounts to replacing one large convolution operation (i.e., one with a large convolution tensor) with a sequence of smaller grouped 3D convolutions, as summarized in Algorithm \ref{alg:kronecker-sequence-conv2d}.
}

The ability to avoid reconstruction at runtime when performing a convolution using any SeKron factorization is the result of the following Theorem:

\begin{restatable}[Linear Mappings with Sequences of Kronecker Products]{theorem}{ksdconv}
Any linear mapping using a given tensor $\tensor{w}$ can be written directly in terms of its Kronecker factors $\tensor{a}_{r_1 \cdots r_k j_1^{(k)} \cdots j_N^{(k)}}^{(k)} \in \R^{a_1^{(k)} \times \cdots \times a_N^{(k)}}$. That is:
\begin{equation}
    \tensor{w}_{i_1 \cdots i_N} \tensor{x}_{i_1 + z_1, \cdots, i_N + z_N} 
    = 
    \tensor{a}_{r_1 j_1^{(1)} \cdots j_N^{(1)}}^{(1)}
    \cdots
    \tensor{a}_{r_1 \cdots r_{S-1} j_1^{(S)} \cdots j_N^{(S)}}^{(S)}
    \tensor{x}_{f(\vec{j}_1) + z_1, \cdots, f(\vec{j}_N) + z_N} 
    \label{eq:linear-mappings-with-kronecker-sequences-scalar-form}
\end{equation}
where $j_n^{(k)}$ terms index Kronecker factors as in  \eqref{eq:kron-tensor-sequences-indexing},  and 
$
    f(\vec{j}_n) = \sum_{k=1}^{S} j_n^{(k)} \prod_{l=k+1}^{S} a_n^{(l)}
$
\end{restatable}
\begin{proof}
See Appendix
\end{proof}

{\setlength{\textfloatsep}{0pt}
\setlength{\floatsep}{0pt}
\setlength{\intextsep}{0pt}
\begin{algorithm}[t]
    \KwIn{$\{\tensor{a}^{(i)}\}_{i=1}^{S}, \tensor{a}^{(i)} \in \R^{r_i \times f_i \times c_i \times Kh_i \times Kw_i} \quad \tensor{x} \in \R^{N \times C \times H \times W}$}
    \KwOut{$\tensor{x} \in \R^{N \times \prod_{k=1}^{S}f_k \times H \times W}$}

    \For{$i \gets S, S-1, \dots, 1$}{
        \eIf{$i==S$}
          {
            $\tensor{x} \gets \textsc{Conv3d}(\textsc{Unsqueeze}(\tensor{x}, 1), \textsc{Unsqueeze}(\tensor{a}^{(i)}, 1)$ \\ 
            \tcc{$\R^{N \times \prod_{k=i+1}^{S}f_k \times r_{i}f_{i} \times \prod_{k=1}^{i-1}c_k \times H \times W } \to \R^{N \times \prod_{k=i}^{S}f_k \times r_{i-1} \times \prod_{k=1}^{i-1}c_k \times H \times W } $} 
            $\tensor{x} \gets\textsc{Reshape}_{1}(\tensor{x})$
          }
          {
            $\tensor{x} \gets \textsc{Conv3d}(\tensor{x}, \tensor{a}^{(i)}, \text{groups}=ri)$ \\
            \tcc{$\R^{N \times \prod_{k=i+1}^{S}f_k \times r_{i}f_{i} \times \prod_{k=1}^{i-1}c_k \times H \times W } \to \R^{N \times \prod_{k=i}^{S}f_k \times r_{i} \times \prod_{k=1}^{i-1}c_k \times H \times W } $} 
            $\tensor{x} \gets\textsc{Reshape}_{2}(\tensor{x})$ 
          }
    }
    \Return $\tensor{x}$
    \caption{Convolution operation using a sequence of Kronecker factors}
    \label{alg:kronecker-sequence-conv2d}
\end{algorithm}}

Using \eqref{eq:linear-mappings-with-kronecker-sequences-scalar-form}, we re-write the projection in \eqref{eq:kseqconv2d-reconstructed-scalar-form} directly in terms of Kronecker factors
{
\setlength{\abovedisplayskip}{0pt}
\setlength{\belowdisplayskip}{0pt}
\begin{multline}
    \tensor{y}_{fxy} = 
    \sum_{\substack{\vec{i},\vec{j},\vec{c}, r_1}}
    \tensor{a}_{r_1 f_1 c_1 i_1 j_1}^{(1)}
    \sum_{r_2} \tensor{a}_{r_1, r_2, f_2, c_2, i_2, j_2}^{(2)}
    \cdots \\
    \sum_{r_{S-1}} 
    \tensor{a}_{r_1 \cdots r_{S-1} f_{N-1} c_{N-1} i_{N-1} j_{N-1}}^{(S-1)}
    \tensor{a}_{r_1 \cdots r_{S-1} f_N c_N i_N j_N}^{(S)}
    \tensor{X}_{f(\vec{c}) ,f(\vec{i})+x,f(\vec{j})+y},
    \label{eq:kseqconv2d-summations-outside-scalar-form}
\end{multline}
where $\vec{i} = (i_1, i_2, \dots, i_N), \; \vec{j} = (j_1, j_2, \dots, j_N), \; \vec{c} = (c_1, c_2, \dots, c_N)$ denote vectors containing indices $i_k, j_k, c_k$ that enumerate over positions in tensors $\tensor{a}^{(k)}$. Finally, exchanging the order of summation separates the convolution as follows: 
\begin{equation}
    \tensor{y}_{fxy} = 
    \sum_{i_1,j_1,c_1,r_1}
    \tensor{a}_{r_1 f_1 c_1 i_1 j_1}^{(1)}
    \cdots 
    \sum_{i_N,j_N,c_N}
    \tensor{a}_{r_1 \cdots r_{S-1} f_N c_N i_N j_N}^{(S)}
    \tensor{X}_{f(\vec{c}), f(\vec{i})+x, f(\vec{j})+y}.
    \label{eq:kseqconv2d-summations-inside-scalar-form}
\end{equation}
Overall, the projection in \eqref{eq:kseqconv2d-summations-inside-scalar-form} can be carried out efficiently using a sequence of grouped 3D convolutions with intermediate reshaping operations as described in Algorithm \ref{alg:kronecker-sequence-conv2d}. }
Refer to Appendix for discussions on universal approximation properties of
NN with weights represented using SeKron.

\subsection{Computational Complexity}
\label{ss:computational-complexity}
In order to decompose a given tensor using our method, the sequence length and the Kronecker factor shapes must be specified. Different selections will lead to different FLOPs, parameters, and latency. Specifically, for the decomposition given by \eqref{eq:ksd-low-rank-tensor-form}
for $\widehat{\tensor{w}} \in \R^{f \times c \times h \times w}$ using factors $\tensor{a}_{r_1 \cdots r_i}^{(i)} \in \R^{f_i \times c_i \times h_i \times w_i}$, the compression ratio (CR) and FLOPs reduction ratio (FR) are given by
\begin{equation}
    \text{CR}=\frac
    {
    \prod_{i=1}^{S}  f_i c_i h_i w_i
    }
    {
    \sum_{i=1}^{S} \prod_{k=1}^{i} \widehat{R}_{k} f_i c_i h_i w_i
    }, \quad
    \text{FR}=\frac
    {
    \prod_{i=1}^{S} f_i c_i h_i w_i
    }
    {
    \sum_{i=1}^{S} \left(\prod_{k=i}^{S}F_{k}\right) \left(\prod_{k=1}^{i}  \widehat{R}_{k}\right) \left(\prod_{k=1}^{i} c_{k}\right)
    h_{i} w_{i}
    }.
    \label{eq:kseqconv2d-cr-fr} \vspace{-1mm}
\end{equation}

Applying SeKron to compress DNN models requires a selection strategy for sequence lengths and factor shapes for each layer in a network. We adopt a simple approach that involves selecting configurations that best match a desired CR while also having a lower latency than the original layer being compressed, as FR may not be a good indicator of runtime speedup in practice.

\section{Experimental Results}

To demonstrate the effectiveness of SeKron for model compression, we evaluate different CNN models on both high-level and low-level computer vision tasks. For image classification tasks, we evaluate WideResNet16 \cite{zagoruyko2016wide} and ResNet50 \cite{he2016deep} models on CIFAR-10 \cite{krizhevsky2012cifar10} and ImageNet \cite{krizhevsky2012imagenet}, respectively. For super-resolution task, we evaluate EDSR-8-128 and SRResNet16 trained on DIV2k\cite{agustsson2017div2k}. Lastly, we discuss the latency of our proposed decomposition method.

\subsection{Image Classification Experiments}

First, we evaluate SeKron to compress WideResNet16-8 \cite{zagoruyko2016wide} for image classification on CIFAR-10.
We evaluate our method against a range of decomposition and pruning approaches at various compression rates. Namely, PCA \cite{zhang2016accelerating} which imposes that filter responses lie approximately on a low-rank subspace; SVD-Energy \cite{alvarez2017compression} which imposes a low-rank regularization into the training procedure;  L-Rank (learned rank selection) \cite{idelbayev2020lowrank} which jointly optimizes over matrix elements and ranks; ALDS \cite{liebenwein2021compressing} which provides a global compression framework that finds optimal layer-wise compressions that lead to an overall desired global compression rate; TR \cite{wang2018wide}; TT \cite{novikov2015tensorizing} as well as two recent pruning approaches FT \cite{li2017pruning} and PFP \cite{liebenwein2020provable}. We note that since ALDS and L-Rank are rank selection frameworks, they can be used on top of other decomposition methods such as SeKron. 

Figure~\ref{fig:wrn16-compression}, shows the CIFAR-10 classification performance drop (i.e., $\Delta$ Top-1) versus compression rates achieved with different model compression methods. As this figure suggests, our approach outperforms all other decomposition and pruning methods, at a variety of compression rates indicated as the percentage of retained parameters once model compression is applied. In Table~\ref{tbl:wrn16-compression} we highlight that at a compression rate of $4\times$ SeKron 
outperforms all other methods. In fact, SeKron has a small accuracy drop of $-0.51$, whereas the next best decomposition method (omitting rank selection approaches) suffers a $-1.27$ drop in accuracy.

\begin{figure}[t]
\vspace{-4mm}
\small
\centering
\begin{minipage}{.52\linewidth}
\centering
\captionsetup{width=1\linewidth}
\captionof{table}{Performance of compressed WideResNet16-8 using various methods on CIFAR-10}
\label{tbl:wrn16-compression}
\begin{tabular}{lcc}
\toprule
Model      & CR & $\Delta$ Top-1 (\%) \\ \midrule
ALDS       & 4.0         & $-$0.73        \\
L-Rank     & 4.0         & $-$3.52        \\ \midrule
FT         & 4.1         & $-$1.50         \\
PFP        & 4.0         & $-$0.94        \\  \midrule
SVD        & 4.0         & $-$4.40         \\
PCA        & 4.0         & $-$2.08        \\
SVD Energy & 4.0         & $-$1.27        \\
TT         & 4.0           & $-$2.86        \\
TR          & 4.0          & $-$0.70       \\
CP          & 4.0           & $-$3.13       \\
Tucker      & 4.0           &  $-$ 1.61         \\
SeKron (Ours)       & 4.1         & \textbf{$-$0.51}        \\
\bottomrule
\end{tabular}
\end{minipage}
\hspace{-0.01\textwidth}%
\begin{minipage}{.48\linewidth}
\begin{tikzpicture}[scale=1] \centering
\begin{axis}[
    xmin = 0, xmax = 60,
    ymin = -9.5, ymax =0.5,
    xtick distance = 20,
    ytick distance = 1,
    grid = both,
    minor tick num = 1,
    major grid style = {lightgray},
    minor grid style = {lightgray!25},
    width = \linewidth,
    height = 0.8\linewidth,
    xlabel = {Parameters Retained (\%)},
    ylabel = {$\Delta$ Top-1 (\%)},
  legend style={at={(1,0)},anchor=south east, font=\tiny},
  legend columns=2,every axis plot/.append style={very thick}]
 
\addplot[red, mark=*] file[skip first] {data-wrn16/kseq3.dat};
\addlegendentry{SeKron (Ours)};
\addplot[blue, dashed, mark=square*] file[skip first] {data-wrn16/alds.dat};
\addlegendentry{ALDS};
\addplot[green, mark=x] file[skip first] {data-wrn16/pca.dat};
\addlegendentry{PCA};
\addplot[violet, mark=otimes*] file[skip first] {data-wrn16/svd-energy.dat};
\addlegendentry{SVD-Energy};
\addplot[orange, dashed, mark=diamond*] file[skip first] {data-wrn16/tt.dat};
\addlegendentry{TT};
\addplot[gray, mark=diamond*] file[skip first] {data-wrn16/tr.dat};
\addlegendentry{TR};
\addplot[purple, mark=star] file[skip first] {data-wrn16/ft.dat};
\addlegendentry{FT};
\addplot[magenta, dashed ,  mark=star] file[skip first] {data-wrn16/pfp.dat};
\addlegendentry{PFP};
\addplot[brown, mark=square*] file[skip first] {data-wrn16/lrank.dat};
\addlegendentry{L-Rank};
\addplot[cyan, mark=square*] file[skip first] {data-wrn16/cp.dat};
\addlegendentry{CP};
\addplot[pink, mark=square*] file[skip first] {data-wrn16/tucker.dat};
\addlegendentry{Tucker};
\end{axis}
\end{tikzpicture}
\captionsetup{width=1\linewidth}
\captionof{figure}{Performance drop of WideResNet16-8 at various compression rates achieved by different methods on CIFAR-10.}
\label{fig:wrn16-compression}
\end{minipage}
\vspace{-3mm}
\end{figure}

\begin{table}[t]
\small
\centering
\caption{Performance of ResNet50 using various compression methods measured on ImageNet. $\ssymbol{2}$ indicates models obtained from compressing baselines with different accuracies, for this reason we report accuracy drops of each model with respect to their own baselines as well. The baselines compared are FSNet \cite{yang2020fsnet}, ThiNet \cite{luo2017thinet} CP \cite{he2017channel} MP \cite{liu2019metapruning} and Binary Kronecker \cite{gamal2022convolutional} }
\label{tbl:resnet50-compression}
\setlength\tabcolsep{5pt} 
\begin{tabular}{@{}lccccc@{}}
\toprule
Method & Type & Params (E+6) / CR & FLOPS (E+9) & CPU (ms) & Top-1 / $\Delta$ Top-1 \\ \midrule
FSNet$\ssymbol{2}$   & Other & 13.9 / 1.8 & - & - & 73.11 / $-$2.0 \\ 
ThiNet$\ssymbol{2}$   & \multirow{3}{*}{Pruning} & 12.4 / 2.1 & - & - & 71.01 / $-$1.9 \\
CP$\ssymbol{2}$   &  &  - / 2.0 & - & - & 73.30 / $-$3.0 \\
MP$\ssymbol{2}$   &  & 10.6 / 2.4 & - & - & 73.40 / $-$3.2  \\
\midrule
Tensor Ring & \multirow{4}{*}{Decomposition} & 13.9 / 1.8  & 2.1 & 105 $\pm$ 2 &  73.30 / $-$2.7  \\
Tensor Train &  & 13.3 / 1.9  & 1.9 &  395 $\pm$ 54  & 73.85 / $-$2.1  \\
Binary Kronecker  &  & 12.0 / 2.1 & - & - & 73.95 / $-$ 2.0  \\
SeKron $S=2$ (Ours)  &  & 12.3 / 2.0   & 2.9  & 125 $\pm$ 3 & 74.66 / $-$1.3 \\
SeKron $S=3$ (Ours)  &  & 13.8 / 1.8  & 2.5 & 133 $\pm$ 4\ & \textbf{74.94} / $-$\textbf{1.1} \\ \midrule
Baseline & Uncompressed & 25.5 / 1.0 & 4.10 & 133 $\pm$ 33 & 75.99 / $-$ 0.0  \\ 
\bottomrule
\end{tabular}
\vspace{-3mm}
\end{table}

\begin{table}[t]
\small
\centering
\vspace{-2mm}
\caption{PSNR (dB) performance of compressed SRResNet16 and EDSR-8-128 models using FBD (with basis-64-16) \cite{li2019learning} and our SeKron}
\label{tbl:sr-models-compression}
\setlength\tabcolsep{4pt} 
\begin{tabular}{@{}clcccccc@{}}
\toprule
\multirow{2}{*}{Model} &
  \multicolumn{1}{l}{\multirow{2}{*}{Method}} &
  \multirow{2}{*}{Params (E+6)} &
  \multirow{2}{*}{CR} &
  \multicolumn{4}{c}{Dataset}  \\ \cmidrule(lr){5-8}
                            & \multicolumn{1}{c}{} &               &              & Set5           & Set14          & B100           & Urban100                   \\ \toprule 
\multirow{5}{*}[-0.8ex]{\rotatebox[origin=c]{90}{SRResNet16}} & Baseline             & 1.54          & 1.0          & 32.03          & 28.5           & 27.52          & 25.88            \\ \cmidrule(l){2-8} 
                            & FBD          & 0.65          & 2.4          & 31.84          & 28.38          & 27.39          & 25.54             \\
                            & SeKron                  & 0.65 & 2.4          & \textbf{31.91} & \textbf{28.42} & \textbf{27.43} & \textbf{25.64}   \\ \cmidrule(l){2-8} 
                            & FBD          & 0.36          & 4.3 & 31.49          & 28.18          & 27.28          & 25.20      \\
                            & SeKron                  & 0.37          & 4.2 & \textbf{31.73} & \textbf{28.32} & \textbf{27.37} & \textbf{25.48}   \\ \toprule
\multirow{5}{*}[-0.8ex]{\rotatebox[origin=c]{90}{EDSR-8-128}} & Baseline             & 3.70          & 1.0          & 32.13          & 28.55          & 27.55          & 26.02            \\ \cmidrule(l){2-8} 
                            & FBD          & 1.62          & 2.3          & 31.80          & 28.34          & 27.40          & 25.54          \\
                            & SeKron                  & 1.50          & 2.5          & 31.79          & 28.34          & 27.39          & 25.52          \\ \cmidrule(l){2-8} 
                            & FBD          & 0.48          & 7.8          & 31.64          & 28.23          & 27.32          & 25.31            \\
                            & SeKron                  & 0.47          & 7.8          & \textbf{31.77} & \textbf{28.32} & \textbf{27.38} & \textbf{25.46}  \\ \bottomrule
\end{tabular}
\vspace{-2mm}
\end{table}

Next, we evaluate SeKron to compress ResNet50 for the image classification task on ImageNet dataset. 
In Table~\ref{tbl:resnet50-compression}, we compare our method to other decomposition methods as well as other compression approaches. Most notably, SeKron outperforms all other decomposition methods as well as the other compression approaches, achieving $74.94\%$ Top-1 accuracy which is ${\sim}1.1\%$ greater than the second highest accuracy achieved by using TT decomposition.  At the same time, the proposed method is approximately $3\times$ faster than TT decomposition on a single CPU.  Table~\ref{tbl:resnet50-compression} implies that SeKron is better suited for model compression that targets edge devices with limited CPUs.

\subsection{Super-Resolution Experiments}
\label{ss:SR-experiments}

We applied our method to SRResNet \cite{ledig2017photo} 
and EDSR \cite{lim2017enhanced} super-resolution models. EDSR is quite a huge network while SRResNet is a middle-level network. Thus, for a fast training, we followed \cite{li2019learning} and trained a lighter version of EDSR with 8 residual blocks and 128 channels per convolution in the residual block denoted as EDSR-8-128. Both of these networks were trained on DIV2K \cite{agustsson2017div2k} dataset that contains 1,000 2K images. We test the networks on four benchmarking datasets; Set5 \cite{bevilacqua2012set5}, Set14 \cite{zeyde2012set14}, B100\cite{martin2001b100}, and Urban100 \cite{huang2015urban100}. In our experiments, we only present the results for the $\times$4 scaling factor since it is more challenging than the $\times$2 super-resolution task.  Table~\ref{tbl:sr-models-compression} presents the performances in terms of PSNR measured on the test images for the models once compressed using SeKron along with the original uncompressed models. 

Among model compression methods, Filter Basis Decomposition (FBD) \cite{li2019learning} has been previously shown to achieve state-of-the-art compression on super-resolution CNNs. Therefore, we compare our model compression results with those obtained using FBD as shown in Table~\ref{tbl:sr-models-compression}.  We highlight that our approach outperforms FBD, on all test datasets when compressing SRResNet16 at similar compression rates. As this table suggests, when compression rate is increased, FBD results in much lower PSNRs for both EDSR-8-128 and SRResNet16 compared to our proposed SeKron.

\subsection{Configuring SeKron Considering Latency and Compression Rate}
\label{ss:configuration-selection}

Using the configuration selection strategy proposed in \ref{ss:computational-complexity}, we find that a small sequence length ($S$) is limited to few achievable candidate configurations (and consequently compression rates) that do not sacrifice latency. This is illustrated in Figure~\ref{fig:fast-configurations} for $S=2$ where targeting a CPU latency less than 5 ms and a compression ratio less than 10$\times$ leaves only 3 options for compression. In contrast, increasing the sequence length to $S=3$ leads to a wider range of achievable compression rates (i.e., 129 configurations). Despite the flexibility they provide, large sequence lengths lead to an exponentially larger number of candidate configurations and time-consuming generation of all their runtimes. For this reason, unless otherwise stated, we opted to use $S=3$ in all the above-mentioned experiments as it provided a suitable range of compression rates and a manageable search space. 

As an example, in Table~\ref{tbl:sr-models-compression-latency} we compress EDSR-8-128 using a compression rate of $\text{CR}=2.5\times$, by selecting configurations for each layer that satisfy the desired $\text{CR}$ while simultaneously resulting in a speedup. This led to an overall model speedup of 124ms (compressed) vs. 151ms (uncompressed).

\begin{figure}[t]
\centering
\begin{minipage}{.4\linewidth}
\small
\centering
\captionof{table}{CPU latency (ms) for uncompressed (baseline) and compressed SRResNet16 and EDSR-8-128 models using SeKron}
\label{tbl:sr-models-compression-latency}
\setlength\tabcolsep{4pt} 
\begin{tabular}{@{}llccc@{}}
\toprule
 Model &
  Method &
  CR &
  CPU (ms) \\ 
                            \midrule
\multirow{3}{*}[-0.0ex]{SRResNet16} & Baseline                      & 1.0            & 72 $\pm$ 3   \\
                            & SeKron               & 2.4          &  70 $\pm$ 5   \\
                            & SeKron                          & 4.2 & 70 $\pm$ 2   \\ \midrule
\multirow{3}{*}[-0.0ex]{EDSR-8-128} & Baseline                & 1.0        & 151 $\pm$ 8  \\
                            & SeKron                   & 2.5           & 124 $\pm$ 4  \\
                            & SeKron              & 7.8        & 131 $\pm$ 9  \\ \bottomrule
\end{tabular}
\end{minipage}%
\hspace{0.01\textwidth}%
\begin{minipage}{.59\linewidth}
\centering
\includegraphics[width=0.6\linewidth,trim={0.0cm 0.0cm 0.0cm 0.0cm},clip]{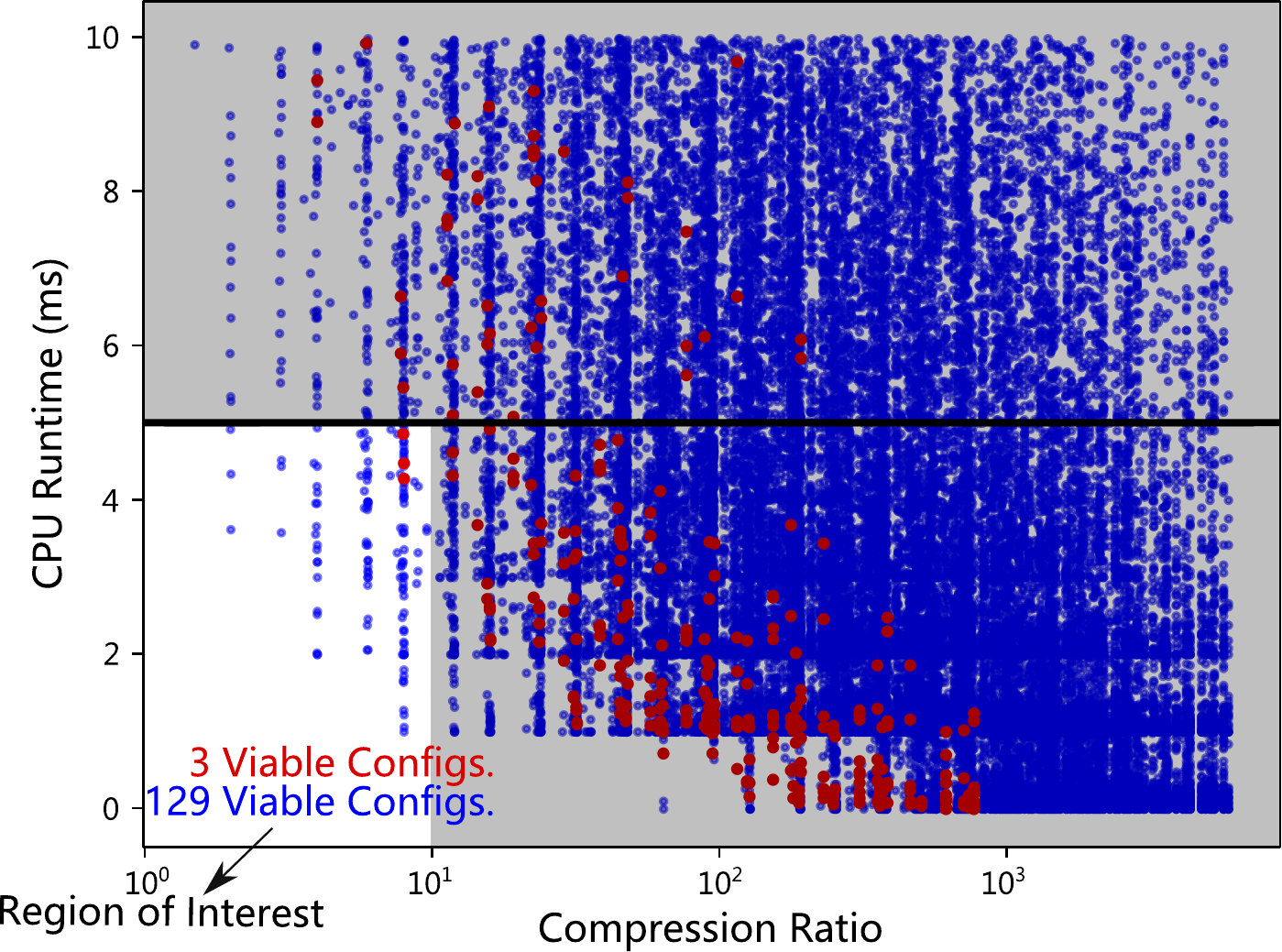}
\captionsetup{width=1\linewidth}
\vspace{-2mm}
\captionof{figure}{CPU latency for candidate configurations obtained using SeKron on a tensor $\tensor{w} \in \R^{512 \times 512 \times 3 \times 3}$ with $S=2$ (red) and $S=3$ (blue), aiming for a speedup (e.g., $<5$ ms) and a typical compression rate (e.g., $<10\times$).
}
\label{fig:fast-configurations}
\end{minipage}
\vspace{-5mm}
\end{figure}

\section{Conclusions}
We introduced SeKron, a tensor decomposition approach using sequences of Kronecker products. SeKron allows for a wide variety of factorization structures to be achieved, while, crucially, sharing the same compression and convolution algorithms. Moreover, SeKron has been shown to generalize popular decomposition methods such as TT, TR, CP and Tucker. Thus, it mitigates the need for time-consuming development of customized convolution algorithms. Unlike other decomposition methods, SeKron is not limited to a single factorization structure, which leads to  improved compressions and reduced runtimes on different hardware. Leveraging SeKron's flexibility, we find efficient factorization structures that outperform previous decomposition methods on various image classification and super-resolution tasks.

\bibliography{bibfile}
\bibliographystyle{iclr2023_conference}

\section{Appendix}
\input{appendix}

\end{document}

%% file: math_commands.tex
\usepackage{amsmath,amsfonts,bm}

\def\eqref#1{equation~\ref{#1}}

\def\1{\bm{1}}

\DeclareMathAlphabet{\mathsfit}{\encodingdefault}{\sfdefault}{m}{sl}
\SetMathAlphabet{\mathsfit}{bold}{\encodingdefault}{\sfdefault}{bx}{n}

\newcommand{\R}{\mathbb{R}}



%% file: marawan_math_commands.tex
\usepackage[mathscr]{euscript}
\usepackage{amssymb}
\usepackage{amsmath}

\renewcommand{\vec}[1]{\MakeLowercase{\mathbf{#1}}}
\newcommand{\mat}[1]{\MakeUppercase{\mathbf{#1}}}

\usepackage{tensor}
\let\tensorp\tensor
\renewcommand{\tensor}[2][]{
    \tensorp{\boldsymbol{\mathcal{\MakeUppercase{#2}}}}{#1}
}

\renewcommand{\mod}[1]{\hspace{1pt}\text{mod}\hspace{1pt}#1}
\def \R{{\rm I\!R}}
\def \N{{\rm I\!N}}

\usepackage{amsthm}

\newtheorem{theorem}{Theorem}

\usepackage{thmtools,thm-restate}

\usepackage[ruled]{algorithm2e}
\SetKwComment{Comment}{/* }{ */}
\newlength\mylen

\newcommand{\nonl}{\renewcommand{\nl}{\let\nl\oldnl}}

\usepackage{booktabs}
\usepackage{float}
\usepackage{subcaption}

%% file: intro.tex

\vspace{-1mm}
Deep learning models have introduced new state-of-the-art solutions to both high-level computer vision problems \cite{he2016deep,ren2015faster}, and low-level image processing tasks \cite{wang2018esrgan,
schuler2015learning,
kokkinos2018deep} through convolutional neural networks (CNNs). Such models are obtained at the expense of millions of training parameters that come along deep CNNs making them computationally intensive. As a result, many of these models are of limited use as they are challenging to deploy on resource-constrained edge devices.
Compared with neural networks for high-level computer vision tasks (e.g., ResNet-50 \cite{he2016deep}), models for low-level imaging problems such as single image super-resolution have much a higher computational complexity due to the larger feature map sizes. Moreover, they are typically infeasible to run on cloud computing servers. Thus, their deployment on edge devices is even more critical.

In recent years an increasing trend has begun in reducing the size of state-of-the-art CNN backbones through efficient architecture designs such as Xception \cite{chollet2017xception}, MobileNet \cite{howard2019searching}, ShuffleNet \cite{zhang2018shufflenet}, and EfficientNet \cite{tan2019efficientnet}, to name a few. On the other hand, there have been studies demonstrating significant redundancy in the parameters of large CNN models, implying that in theory the number of model parameters can be reduced while maintaining performance \cite{denil2013predicting}. 
These studies provide the basis
for the development of many model compression techniques such as pruning \cite{yang2020learning}, quantization \cite{hubara2017quantized},  knowledge distillation \cite{hinton2015distillation}, and tensor decomposition \cite{phan2020stable}.


Tensor decomposition methods such as Tucker \citep{deok2016compression}, Canonical Polyadic (CP) \citep{lebedev2014speeding}, Tensor-Train (TT) \citep{novikov2015tensorizing} and Tensor-Ring (TR) \citep{wang2018wide}  rely on finding low-rank approximations of tensors under some imposed factorization structure as illustrated in Figure~\ref{fig:tensor-networks-and-inefficiency-of-smaller-sequneces}a. 
In practice, some structures are more suitable than others when decomposing tensors. Choosing from a limited set of factorization structures can lead to sub-optimal compressions as well as lengthy runtimes depending on the hardware. This limitation can be alleviated by reshaping tensors prior to their compression to improve performance as shown in \cite{garipov2016ultimate}. However, this approach requires time-consuming development of customized convolution algorithms.


\begin{figure}[t]
\begin{minipage}{0.60\linewidth}
\vspace{-1.5mm}
\centering
\begin{minipage}{0.1\linewidth}
\captionsetup{labelformat=empty}
\caption*{(a)}
\end{minipage}
\begin{minipage}{0.89\linewidth}
\begin{subfigure}[b]{0.3\linewidth}
\resizebox{\linewidth}{!}{
\begin{tikzpicture}
    \node[circle,draw=black, fill=blue!10,inner sep=1pt, minimum size=1cm] (g) at (0,0) {\large $\tensor{g}$};
    
    \node[circle,draw=black, fill=blue!10,inner sep=1pt, minimum size=1cm] (u1) [above = of g] {\large $\tensor{a}^{(1)}$};
    \node[circle] (f) [above = 0.25cm of u1] {\large $f$};
    
    \node[circle,draw=black, fill=blue!10,inner sep=1pt, minimum size=1cm] (u2) [right = of g] {\large $\tensor{a}^{(1)}$};
    \node[circle] (c) [right = 0.25cm of u2] {\large $c$};
    
    \node[circle,draw=black, fill=blue!10,inner sep=1pt, minimum size=1cm] (u3) [below = of g] {\large $\tensor{a}^{(3)}$};
    \node[circle] (h) [below = 0.25cm of u3] {\large $h$};
    
    \node[circle,draw=black, fill=blue!10,inner sep=1pt, minimum size=1cm] (u4) [left = of g] {\large $\tensor{a}^{(4)}$};
    \node[circle] (w) [left = 0.25cm of u4] {\large $w$};

    \path [draw,-] (g) -- (u1) node [midway, left] { };
    \path [draw,-] (g) -- (u2) node [midway, above] { };
    \path [draw,-] (g) -- (u3) node [midway, right] { };
    \path [draw,-] (g) -- (u4) node [midway, below] { };
    
    \path [draw,-] (u1) edge (f);
    \path [draw,-] (u2) edge (c);
    \path [draw,-] (u3) edge (h);
    \path [draw,-] (u4) edge (w);

\end{tikzpicture}}
\captionsetup{labelformat=empty}
\vspace{-5mm}
\caption*{\vspace{-3mm}Tucker}
\end{subfigure}
\hspace{0em}%
\begin{subfigure}[b]{0.3\linewidth}
\resizebox{\linewidth}{!}{
\begin{tikzpicture}
    \node[circle,draw=black, fill=black,inner sep=1pt, minimum size=0.2cm] (g) at (0,0) {};
    
    \node[circle,draw=black, fill=blue!10,inner sep=1pt, minimum size=1cm] (u1) [above = of g] {\large $\tensor{a}^{(1)}$};
    \node[circle] (f) [above = 0.25cm of u1] {\large $f$};
    
    \node[circle,draw=black, fill=blue!10,inner sep=1pt, minimum size=1cm] (u2) [right = of g] {\large $\tensor{a}^{(1)}$};
    \node[circle] (c) [right = 0.25cm of u2] {\large $c$};
    
    \node[circle,draw=black, fill=blue!10,inner sep=1pt, minimum size=1cm] (u3) [below = of g] {\large $\tensor{a}^{(3)}$};
    \node[circle] (h) [below = 0.25cm of u3] {\large $h$};
    
    \node[circle,draw=black, fill=blue!10,inner sep=1pt, minimum size=1cm] (u4) [left = of g] {\large $\tensor{a}^{(4)}$};
    \node[circle] (w) [left = 0.25cm of u4] {\large $w$};

    \path [draw,-] (g) -- (u1) node [midway, left] { };
    \path [draw,-] (g) -- (u2) node [midway, above] { };
    \path [draw,-] (g) -- (u3) node [midway, right] { };
    \path [draw,-] (g) -- (u4) node [midway, below] { };
    
    \path [draw,-] (u1) edge (f);
    \path [draw,-] (u2) edge (c);
    \path [draw,-] (u3) edge (h);
    \path [draw,-] (u4) edge (w);

\end{tikzpicture}}
\captionsetup{labelformat=empty}
\vspace{-5mm}
\caption*{\vspace{-3mm}CP}
\end{subfigure}
\hspace{0em}%
\begin{subfigure}[b]{0.23\linewidth}
\resizebox{\linewidth}{!}{
\begin{tikzpicture}
    \node[circle,draw=black,fill=blue!10,inner sep=1pt, minimum size=0.6cm] (a1) at (0,0) {\large $\tensor{a}^{(1)}$};
    \node [circle] (f) [above= 0.25cm of a1] {\large $f$};
    \path [draw,-] (f) edge (a1);
    
    \node[circle,draw=black,fill=blue!10,inner sep=1pt, minimum size=0.6cm] (a2) [below right=of a1] {\large $\tensor{a}^{(2)}$};
    \node [circle] (c) [above= 0.25cm of a2] {\large $c$};
    \path [draw,-] (a2) edge (c);
    
    \node[circle,draw=black,fill=blue!10,inner sep=1pt, minimum size=0.6cm] (a3) [below left=of a2] {\large $\tensor{a}^{(3)}$};
    \node [circle] (h) [above= 0.25cm of a3] {\large $h$};
    \path [draw,-] (h) edge (a3);
    \node[circle,draw=black,fill=blue!10,inner sep=1pt, minimum size=0.6cm] (a4) [above left =of a3] {\large $\tensor{a}^{(4)}$};
    \node [circle] (w) [above= 0.25cm of a4] {\large $w$};
    \path [draw,-] (w) edge (a4);
    
    \path [draw,-] (a1) -- (a2) node [midway, above=8pt, right] { };
    \path [draw,-] (a2) -- (a3) node [midway, below=8pt, right] { };
    \path [draw,-] (a3) -- (a4) node [midway, below=8pt, left] { };
    \path [draw,-] (a4) -- (a1) node [midway, above=8pt, left] { };
\end{tikzpicture}}
\captionsetup{labelformat=empty}
\vspace{-5mm}
\caption*{\vspace{-3mm}TR}
\end{subfigure}

 \vspace{1mm}    
 \hspace{5mm}
\begin{subfigure}[b]{0.38\linewidth}
\resizebox{\linewidth}{!}{
\begin{tikzpicture}
    \node[circle,draw=black,fill=blue!10,inner sep=1pt, minimum size=0.6cm] (a1) at (0,0) {\large $\tensor{a}^{(1)}$};
    \node [circle] (f) [above= 0.25cm of a1] {\large $f$};
    \path [draw,-] (f) edge (a1);
    
    \node[circle,draw=black,fill=blue!10,inner sep=1pt, minimum size=0.6cm] (a2) [right=of a1] {\large $\tensor{a}^{(2)}$};
    \node [circle] (c) [above= 0.25cm of a2] {\large $c$};
    \path [draw,-] (a2) edge (c);
    
    \node[circle,draw=black,fill=blue!10,inner sep=1pt, minimum size=0.6cm] (a3) [right=of a2] {\large $\tensor{a}^{(3)}$};
    \node [circle] (h) [above= 0.25cm of a3] {\large $h$};
    \path [draw,-] (h) edge (a3);

    \node[circle,draw=black,fill=blue!10,inner sep=1pt, minimum size=0.6cm] (a4) [right=of a3] {\large $\tensor{a}^{(4)}$};
    \node [circle] (w) [above= 0.25cm of a4] {\large $w$};
    \path [draw,-] (w) edge (a4);
    
    \path [draw,-] (a1) -- (a2) node [midway, above] { };
    \path [draw,-] (a2) -- (a3) node [midway, above] { };
    \path [draw,-] (a3) -- (a4) node [midway, above] { };
\end{tikzpicture}}
        \vspace{-4mm}
\captionsetup{labelformat=empty}
\caption*{\vspace{-4mm}TT}
\end{subfigure}
\begin{subfigure}[b]{0.38\linewidth}
\resizebox{\linewidth}{!}{
    
    
    
    
    
    
    
\begin{tikzpicture}[mystyle/.style={draw,shape=circle,fill=blue!10, minimum size=0.6cm, inner sep=1pt}]
\def\ngon{5}

\node[regular polygon, regular polygon sides=\ngon, minimum size=3cm] (p) {};
\foreach\x in {1,...,\ngon}{
    \ifthenelse{\x = 2}{
        \node[mystyle] (p\x) at (p.corner \x){$\tensor{a}^{(S)}$};
        }
        { 
        \node[mystyle] (p\x) at (p.corner \x){$\tensor{a}^{(\x)}$};
        }
}
\foreach\x in {1,...,\numexpr\ngon-1\relax}{
    \ifthenelse{\NOT \x = 2}{
        \foreach\y in {\x,...,\ngon}{
            \draw (p\x) -- (p\y);
        }
    }
    {
        \foreach\y in {\x,...,\ngon}{
            \ifthenelse{\NOT \y = 3}{
                \draw (p\x) -- (p\y);
                }{}
        }
    }
}

\node (dots) at ($(p.corner 2)!0.5!(p.corner 3)$) {\Large{\rotatebox{108}{$\mathbf{\cdots}$}}};
\draw (p3) -- (dots);
\draw (p2) -- (dots);



\node[circle] (m1) [above = 0.3cm of p1] {};
\node[circle] (c1) [left= -0.2cm of m1] {$\color{red} c_1$};
\node[circle] (f1) [left= -0.2cm of c1] {$\color{red} f_1$};
\node[circle] (h1) [right= -0.2cm of m1] {$\color{red} h_1$};
\node[circle] (w1) [right= -0.2cm of h1] {$\color{red} w_1$};

\node[circle] (m2) [above right= 0.1cm and 0.4cm of p5] {};
\node[circle] (c2) [above left= -0.2cm and -0.2cm of m2] {$\color{red} c_2$};
\node[circle] (f2) [above left= -0.2cm and -0.2cm of c2] {$\color{red} f_2$};
\node[circle] (h2) [below right= -0.2cm and -0.2cm of m2] {$\color{red} h_2$};
\node[circle] (w2) [below right= -0.2cm and -0.2cm of h2] {$\color{red} w_2$};

\node[circle] (m3) [below right= 0.1cm and 0.4cm of p4] {};
\node[circle] (c3) [below left= -0.2cm and -0.2cm of m3] {$\color{red} c_3$};
\node[circle] (f3) [below left= -0.2cm and -0.2cm of c3] {$\color{red} f_3$};
\node[circle] (h3) [above right= -0.2cm and -0.2cm of m3] {$\color{red} h_3$};
\node[circle] (w3) [above right= -0.2cm and -0.2cm of h3] {$\color{red} w_3$};

\node[circle] (m4) [below left= 0.1cm and 0.4cm of p3] {};
\node[circle] (c4) [above left= -0.2cm and -0.2cm of m4] {$\color{red} c_4$};
\node[circle] (f4) [above left= -0.2cm and -0.2cm of c4] {$\color{red} f_4$};
\node[circle] (h4) [below right= -0.2cm and -0.2cm of m4] {$\color{red} h_4$};
\node[circle] (w4) [below right= -0.2cm and -0.2cm of h4] {$\color{red} w_4$};

\node[circle] (mS) [above left= 0.1cm and 0.4cm of p2] {};
\node[circle] (cS) [below left= -0.2cm and -0.2cm of mS] {$\color{red} c_S$};
\node[circle] (fS) [below left= -0.2cm and -0.2cm of cS] {$\color{red} f_S$};
\node[circle] (hS) [above right= -0.2cm and -0.2cm of mS] {$\color{red} h_S$};
\node[circle] (wS) [above right= -0.2cm and -0.2cm of hS] {$\color{red} w_S$};


\path [draw,-] (p1) edge (f1);
\path [draw,-] (p1) edge (c1);
\path [draw,-] (p1) edge (h1);
\path [draw,-] (p1) edge (w1);

\path [draw,-] (p5) edge (f2);
\path [draw,-] (p5) edge (c2);
\path [draw,-] (p5) edge (h2);
\path [draw,-] (p5) edge (w2);

\path [draw,-] (p4) edge (f3);
\path [draw,-] (p4) edge (c3);
\path [draw,-] (p4) edge (h3);
\path [draw,-] (p4) edge (w3);

\path [draw,-] (p3) edge (f4);
\path [draw,-] (p3) edge (c4);
\path [draw,-] (p3) edge (h4);
\path [draw,-] (p3) edge (w4);

\path [draw,-] (p2) edge (fS);
\path [draw,-] (p2) edge (cS);
\path [draw,-] (p2) edge (hS);
\path [draw,-] (p2) edge (wS);

\end{tikzpicture}
}
\captionsetup{labelformat=empty}
       \vspace{-7mm}
       \caption*{\vspace{-4mm}SeKron}
\end{subfigure}
\end{minipage}
\end{minipage}
\begin{minipage}{0.4\linewidth}
\begin{minipage}{0.1\linewidth}
\vspace{-0.7cm}
\captionsetup{labelformat=empty}
\caption*{(b)}
\end{minipage}
\begin{minipage}{0.88\linewidth}
\begin{subfigure}[b]{\linewidth}
\includegraphics[width=\linewidth]{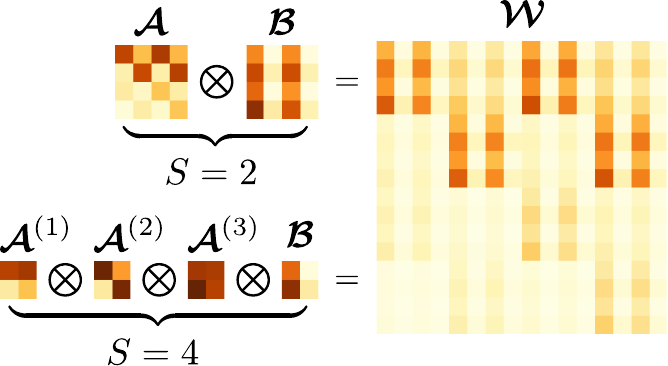}
\captionsetup{labelformat=empty}
\caption{SeKron with two sequence lengths}
\end{subfigure}
\end{minipage}
\end{minipage}
\caption{(a): Tensor network diagrams of various decomposition methods for a 4D convolution tensor $\tensor{w} \in \R^{F \times C \times K_{h} \times K_{w}}$. Unlike all other decomposition methods where $f,c,h,w$ index over \textbf{fixed} dimensions (i.e., dimensions of $\tensor{w}$), SeKron is flexible in its factor dimensions, with ${\color{red} f_k}, {\color{red}c_k}, {\color{red}h_k}, {\color{red}w_k}, \forall k \in \{1,...,S\}$ indexing over \textbf{variable} dimension choices, as well as its sequence length $S$. Thus, it allows for a wide range of factorization structures to be achieved. 
(b): Example of a 16 $\times$ 16 tensor $\tensor{w}$ that can be more efficiently represented using a sequence of four Kronecker factors (requiring \textbf{16 parameters}) in contrast to using a sequence length of two (requiring \textbf{32 parameters}). \vspace{-0.5cm}}
\label{fig:tensor-networks-and-inefficiency-of-smaller-sequneces}
\end{figure}

We propose SeKron, a novel tensor decomposition method offering a wide range of factorization structures that share the \emph{same} efficient convolution algorithm.
Our method is inspired by approaches based on the Kronecker Product Decomposition \cite{thakker2019compressing,gamal2022convolutional}. Unlike other decomposition methods, Kronecker Product Decomposition generalizes the product of smaller factors from vectors
and matrices to a range of tensor shapes, thereby exploiting local redundancy between arbitrary slices of multi-dimensional weight tensors. SeKron represents tensors using \emph{sequences} of Kronecker products to compress convolution tensors in CNNs. We show that using sequences of Kronecker products unlocks a wide range of factorization structures and generalizes other decomposition methods such as Tensor-Train (TT), Tensor-Ring (TR),  Canonical Polyadic (CP) and Tucker, under the same framework.
Sequences of Kronecker products also have the potential to exploit local redundancies using far fewer parameters as illustrated in the example in Figure~\ref{fig:tensor-networks-and-inefficiency-of-smaller-sequneces}b.
By performing the convolution operation using each of the Kronecker factors independently, the number of parameters, computational intensity, and runtime are reduced, simultaneously. 
Leveraging the flexibility SeKron, we find efficient factorization structures that outperform existing decomposition methods on various image classification and low-level image processing super-resolution tasks. In summary, our contributions are: \vspace{-2mm}
\begin{itemize}
    \item Introducing SeKron, a novel tensor decomposition method based on sequences of Kronecker products that allows for a wide range of factorization structures and generalizes other decomposition methods such as TT, TR, CP and Tucker.
    \item Providing a solution to the problem of finding the summation of sequences of Kronecker products between factor tensors that best approximates the original tensor.
    \item Deriving a single convolution algorithm shared by all factorization structures achievable by SeKron, utilized as compressed convolutional layers in CNNs.
    \item Improving the state-of-the-art of low-rank model compression on image classification (high-level vision) benchmarks such as ImageNet and CIFAR-10, as well as super-resolution (low-level vision) benchmarks such as Set4, Set14, B100 and Urban100.
\end{itemize}

%% file: related_work.tex
\vspace{-2mm}
\textbf{Sparsification.}
Different components of DNNs, such as weights \cite{han2015learning,han2015deep}, convolutional filters \cite{he2018soft,luo2017thinet} and feature maps \cite{he2017channel,zhuang2018discrimination} can be sparse. The sparsity can be enforced using sparsity-aware regularization \cite{liu2015sparse,zhou2016less} or pruning techniques \cite{luo2017thinet,han2015learning}.
Many pruning methods \cite{luo2017thinet,zhang2018systematic} aim for a high compression ratio and accuracy regardless of the structure of the sparsity. Thus, they often suffer from imbalanced workload caused by irregular memory access.
Hence, several works aim at zeroing out structured groups of DNN components through more hardware friendly approaches \cite{wen2016learning}.

\textbf{Quantization.}
The computation and memory complexity of DNNs can be reduced by quantizing model parameters into lower bit-widths; wherein the majority of research works use fixed-bit quantization. For instance, the methods proposed in  \cite{gysel2018ristretto,louizos2018relaxed} use fixed 4 or 8-bit quantization. 
Model parameters have been quantized even further into ternary \cite{li2016ternary,zhu2016trained} and binary \cite{courbariaux2015binaryconnect,rastegari2016xnor,courbariaux2016binarized}, 
 representations.
 These methods often achieve low performance even with unquantized activations \cite{li2016ternary}. Mixed-precision approaches, however, achieve more competitive performance as shown in \cite{uhlich2019mixed} where the bit-width for each
layer is determined in an adaptive manner. 
Also, choosing a uniform \cite{jacob2018quantization} or nonuniform \cite{han2015deep,tang2017train,zhang2018lq} quantization
interval has important effects on the compression rate and the acceleration.

\textbf{Tensor Decomposition.}
Tensor decomposition approaches are based on factorizing weight tensors into  smaller tensors to reduce model sizes \cite{yin2021towards}. 
Singular value decomposition (SVD) applied on matrices as a 2-dimensional instance of tensor decomposition is used as one of the pioneering approaches to perform model compression \cite{jaderberg2014speeding}. 
Other classical high-dimensional tensor decomposition
methods, such as Tucker \cite{tucker1963implications} and CP decomposition \cite{harshman1970foundations}, are also adopted to perform model compression. 
However, using these methods often leads to
significant accuracy drops 
\cite{kim2015compression,lebedev2014speeding,phan2020stable}.
The idea of reshaping weights
of fully-connected layers into high-dimensional tensors and
representing them in TT format \cite{oseledets2011tensor}
was extended to CNNs in \cite{garipov2016ultimate}. For multidimensional tensors,
TR decomposition \cite{wang2018wide} has become a more popular option than TT \cite{wang2017efficient}. Subsequent filter basis decomposition works polished these approaches using a shared filter basis. They have been proposed for low-level computer vision tasks such as single image super-resolution in \cite{li2019learning}.
Kronecker factorization is another approach to replace the weight tensors
within fully-connected and convolution layers 
\citep{zhou2015exploiting}. The rank-1 Kronecker product representation limitation of this approach is alleviated in \cite{gamal2022convolutional}. 
The compression rate in \cite{gamal2022convolutional} is determined by both the rank and factor dimensions. For a fixed rank, the maximum compression is achieved by selecting dimensions for each factor that are closest to the square root of the original tensors' dimensions. This leads to representations with more parameters than those achieved using sequences of Kronecker products as shown in Fig. \ref{fig:tensor-networks-and-inefficiency-of-smaller-sequneces}b.

There has been extensive research on tensor decomposition through characterizing global correlation of tensors \cite{zheng2021fully}, extending CP to non-Gaussian data \cite{hong2020generalized},
employing augmented decomposition loss functions \cite{afshar2021swift}, 
etc. for different applications. Our main focus in this paper is on the ones used for NN compression. 

\textbf{Other Methods} 
NNs can also be compressed using Knowledge Distillation (KD) where a large pre-trained network known as teacher is used
to train a smaller student network \cite{mirzadeh2020improved, heo2019knowledge}.
Sharing weights in a more structured manner can be another model compression approach as FSNet \citep{yang2020fsnet} which shares filter weights across spatial locations or
ShaResNet \cite{boulch2018reducing} which reuses convolutional mappings within the same scale level.
Designing lightweight CNNs \cite{sandler2018mobilenetv2,iandola2016squeezenet,chollet2017xception,howard2019searching,zhang2018shufflenet,tan2019efficientnet} is another direction orthogonal to the aforementioned approaches. 
\vspace{-2mm}

%% file: appendix.tex
\ksdecomposition*
\begin{proof}
First, we define intermediate tensors
\begin{multline}
    \tensor{b}_{r_1 \cdots r_k}^{(k)} \triangleq 
    \sum_{r_{k+1}}^{R_{k+1}} 
    \tensor{a}_{r_1 \cdots  r_{k+1}}^{(k+1)} \otimes
    \sum_{r_{k+2}}^{R_{k+2}} 
    \tensor{a}_{r_1 \cdots r_{k+2}}^{(k+2)} \otimes
    \cdots
    \otimes
    \sum_{r_{S-1}}^{R_{S-1}}
    \tensor{a}_{r_1  \cdots r_{S-1}}^{(S-1)}
    \otimes
    \tensor{a}_{r_1 \cdots r_{S-1}}^{(S)}
    \tag{\ref{eq:intermediate-tensors-full-rank} revisited}
\end{multline}
Then the reconstruction error can be written as
\begin{equation}
    \Bigg\|\tensor{w}_{r_1 \cdots r_{k-1}}^{(k)} - \sum_{r_k=1}^{\widehat{R}_k} \tensor{a}^{(k)}_{r_1 \cdots r_k} \otimes
    \tensor{b}^{(k)}_{r_1 \cdots r_k} \Bigg\|_{\mathrm{F}}^{2}
    \label{eq:appendix-reconstruction-intermediate-tensors}
\end{equation}
where $\tensor{w}^{(1)}$ is the initial tensor being decomposed. As described in Section \ref{ss:kronecker-sequence-decomposition}, using reshaping operations
\begin{equation}
\mat{w}_{r_1 \cdots r_{k-1}}^{(k)} = \textsc{Mat}(\textsc{Unfold}(\tensor{w}_{r_1 \cdots r_{k-1}}^{(k)}, \vec{d}_{\tensor{b}_{r_1 \cdots r_k}^{(k)}})),
\tag{\ref{eq:reshaping-operations-w} revisited}
\end{equation}
\begin{equation}
\vec{a}_{r_1  \cdots r_k}^{(k)} =  \textsc{Vec}(\textsc{Unfold}(\tensor{a}_{r_1 \cdots r_{k}}^{(k)}, \vec{d}_{\tensor{i}_{\tensor{a}_{r_1 \cdots r_k}^{(k)}}})), \; \;
\vec{b}_{r_1 \cdots r_{k}}^{(k)} = \textsc{Vec}(\tensor{b}_{r_1 \cdots r_k}^{(k)}),
\tag{\ref{eq:reshaping-operations-a-b} revisited}
\end{equation}
that preserve the sum of squares allows us to equivalently write the reconstruction error as
\begin{equation}
    \Bigg\|\mat{w}_{r_1 \cdots r_{k-1}}^{(k)} - \sum_{r_k=1}^{\widehat{R}_k} \vec{a}^{(k)}_{r_1 \cdots r_k} 
    \vec{b}^{(k)\top}_{r_1 \cdots r_k} \Bigg\|_{\mathrm{F}}^{2}.
    \label{eq:appendix-reconstruction-intermediate-vectors}
\end{equation}
Now consider the singular value decomposition of matrix
$\mat{w}_{r_1 \cdots r_{k-1}}^{(k)}$ 
and let 
$\vec{u}_{r_1 \cdots r_k}^{(k)}, \vec{v}_{r_1 \cdots r_k}^{(k)}$ 
denote its left and right singular vectors, respectively (with the right singular vector scaled according to its corresponding singlar value). Set $\vec{a}_{r_1 \cdots r_k}^{(k)} = \vec{u}_{r_k}^{(k)}$ and define and define error terms
\begin{equation}
    \vec{\delta}_{r_1 \cdots r_k}^{(k)} = 
    \vec{v}_{r_1 \cdots r_k}^{(k)} - \vec{b}_{r_1 \cdots r_{k}}^{(k)}, \quad
    \epsilon_{r_1 \cdots r_k}^{(k)} = \|\vec{\delta}_{r_1 \cdots r_k}^{(k)} \|.
\end{equation}
Expanding out \eqref{eq:appendix-reconstruction-intermediate-vectors} reveals its recursive form
\begin{align}
    \label{eq:appendix-reconstruction-intermediate-vectors-first-expansion}
    \Bigg\|\mat{w}_{r_1 \cdots r_{k-1}}^{(k)} &- \sum_{r_k=1}^{\widehat{R}_k} \vec{a}^{(k)}_{r_1 \cdots r_k} 
    \vec{b}^{(k)\top}_{r_1 \cdots r_k} \Bigg\|_{\mathrm{F}}^{2}
    =
    \Bigg\|\mat{w}_{r_1 \cdots r_{k-1}}^{(k)} - \sum_{r_k=1}^{\widehat{R}_k} \vec{a}^{(k)}_{r_1 \cdots r_k} 
    (\vec{v}_{r_k}^{(k)} - \vec{\delta}_{r_1 \cdots r_k}^{(k)})^\top
    \Bigg\|_{\mathrm{F}}^{2}  \\
    &=
    \Bigg\|\mat{w}_{r_1 \cdots r_{k-1}}^{(k)} - \sum_{r_k=1}^{\widehat{R}_k} \vec{a}^{(k)}_{r_1 \cdots r_k} 
    \vec{v}_{r_1 \cdots r_k}^{(k)\top} +
    \sum_{r_k=1}^{\widehat{R}_k}
    \vec{a}^{(k)}_{r_1 \cdots r_k}\vec{\delta}_{r_1 \cdots r_k}^{(k)\top}
    \Bigg\|_{\mathrm{F}}^{2}  \\
    &\leq
    \Bigg\|\mat{w}_{r_1 \cdots r_{k-1}}^{(k)} - \sum_{r_k=1}^{\widehat{R}_k} \vec{a}^{(k)}_{r_1 \cdots r_k} 
     \vec{v}_{r_1 \cdots r_k}^{(k)\top} \Bigg\|_{\mathrm{F}}^{2}
    +
    \sum_{r_k=1}^{\widehat{R}_k}
    \Bigg\|
    \vec{a}^{(k)}_{r_1 \cdots r_k}\vec{\delta}_{r_1 \cdots r_k}^{(k)\top} \Bigg\|_{\mathrm{F}}^{2} \\
    &\leq
    \Bigg\|\mat{w}_{r_1 \cdots r_{k-1}}^{(k)} - \sum_{r_k=1}^{\widehat{R}_k} \vec{a}^{(k)}_{r_1 \cdots r_k} 
    \vec{v}_{r_1 \cdots r_k}^{(k)\top} \Bigg\|_{\mathrm{F}}^{2}
    +
    \sum_{r_k=1}^{\widehat{R}_k}
    d^{(k)} \epsilon_{r_1 \cdots r_k}^{(k)} \\
    &= 
    \Bigg(
    \sum_{r_k = \widehat{R}_k+1}^{R_k} \sigma_{r_k}^{2}(\mat{W}_{r_1 \cdots r_{k-1}}^{(k)})
    \Bigg)
    +
    \Bigg(
    \sum_{r_k=1}^{\widehat{R}_k}
    d^{(k)} \epsilon_{r_1 \cdots r_k}^{(k)} \Bigg)\\
    &= 
    \sum_{r_k = \widehat{R}_k+1}^{R_k} \sigma_{r_k}^{2}(\mat{W}_{r_1 \cdots r_{k-1}}^{(k)})
    +
    \sum_{r_k=1}^{\widehat{R}_k}
    d^{(k)} \Bigg\|\vec{v}_{r_1 \cdots r_k}^{(k)} - \vec{b}_{r_1 \cdots r_k}^{(k)} \Bigg\|_{\mathrm{F}}^{2} 
    \label{eq:appendix-reconstruction-error-intermediate-vectors-unrolled-once}
\end{align}
where $d^{(k)} \in \mathbb{N}$ is the number of dimensions of vector
$\vec{a}_{r_1 \cdots r_k}^{(k)}$ 
and $R_k$ is the rank of matrix $\mat{w}_{r_1 \cdots r_{k-1}}^{(k)}$, $\sigma_{r_k}(\tensor{w}_{r_1 \cdots r_{k-1}}^{(k)})$ denotes the $r_k^{\text{th}}$ singular value of tensor $\tensor{w}_{r_1 \cdots r_{k-1}}^{(k)}$. By reshaping vectors $\vec{v}_{r_1 \cdots r_k}^{(k)}, \vec{b}_{r_1 \cdots r_k}^{(k)}$ to matrices according to 
\begin{equation}
    \mat{v}_{r_1 \cdots r_k}^{(k)} = \textsc{Mat}\Bigg(\textsc{Unfold}\Bigg(\textsc{Vec}^{-1}\Bigg(\vec{v}_{r_1 \cdots r_k}^{(k)}, \prod_{s=k+1}^{S}\vec{d}^{(s)}\Bigg), \prod_{s=k+2}^{S}\vec{d}^{(s)}\Bigg)\Bigg),
\end{equation}
\begin{equation}
    \mat{b}_{r_1 \cdots r_k}^{(k)} = \textsc{Mat}\Bigg(\textsc{Unfold}\Bigg(\textsc{Vec}^{-1}\Bigg(\vec{b}_{r_1 \cdots r_k}^{(k)}, \prod_{s=k+1}^{S}\vec{d}^{(s)}\Bigg), 
    \prod_{s=k+2}^{S}\vec{d}^{(s)}\Bigg)\Bigg), 
    \label{eq:appendix-reshaping-operations}
\end{equation}
where $\vec{d}^{(s)} = (a_1^{(s)}, \dots, a_N^{(s)})$ describes the dimensions of the $s^{\text{th}}$ factor, we can re-write \eqref{eq:appendix-reconstruction-error-intermediate-vectors-unrolled-once} as
\begin{align}
    \sum_{r_k = \widehat{R}_k+1}^{R_k} & \sigma_{r_k}^{2}(\mat{W}_{r_1 \cdots r_{k-1}}^{(k)})
    +
    \sum_{r_k=1}^{\widehat{R}_k}
    d^{(k)} \Bigg\|\vec{v}_{r_1 \cdots r_k}^{(k)} - \vec{b}_{r_1 \cdots r_k}^{(k)}\Bigg\|_{\mathrm{F}}^{2} \\
    &=
    \sum_{r_k = \widehat{R}_k+1}^{R_k} \sigma_{r_k}^{2}(\mat{W}_{r_1 \cdots r_{k-1}}^{(k)})
    +
    \sum_{r_k=1}^{\widehat{R}_k}
    d^{(k)} \Bigg\|\mat{v}_{r_1 \cdots r_k}^{(k)} - \mat{b}_{r_1 \cdots r_k}^{(k)}\Bigg\|_{\mathrm{F}}^{2} \\
    &= 
    \sum_{r_k = \widehat{R}_k+1}^{R_k} \sigma_{r_k}^{2}(\mat{W}_{r_1 \cdots r_{k-1}}^{(k)})
    +
    \sum_{r_k=1}^{\widehat{R}_k}
    d^{(k)} \Bigg\|\mat{v}_{r_1 \cdots r_k}^{(k)} - \sum_{r_{k+1}=1}^{\widehat{R}_{k+1}} \vec{a}_{r_1 \cdots r_{k+1}}^{(k+1)} \vec{b}_{r_1 \cdots r_{k+1}}^{(k+1)\top}\Bigg\|_{\mathrm{F}}^{2}.
\end{align}
The last line reveals the recursive nature of the formula (compare with \eqref{eq:appendix-reconstruction-intermediate-vectors-first-expansion}). Unrolling the recursive formula for $k=1, \dots, S-1$, by setting $\mat{w}_{r_1 \cdots r_{k}}^{(k+1)} \gets \mat{v}_{r_1 \cdots r_k}^{(k)}$, leads to the following formula for the reconstruction error:
\begin{multline}
\varepsilon_{\text{SeKron}}(\mat{W}, \vec{r}, \mat{D}) 
    =
    \sum_{r_1 = \widehat{R}_1 + 1}^{R_1} \sigma_{r_1}^{2}(\mat{W}^{(1)})
    +
    d^{(1)}\sum_{r_1=1}^{\widehat{R}_1}\sum_{r_2 = \widehat{R}_2 + 1}^{R_2}
    \sigma_{r_2}^{2}(\mat{w}_{r_1}^{(2)})
    + \cdots \\
    + d^{(1)}  d^{(2)} \cdots d^{(S-2)}
    \sum_{r_1, r_2, \dots, r_{S-2}=1}^{\widehat{R}_1, \cdots, \widehat{R}_{S-2}}
    \sum_{r_{S-1} = \widehat{R}_{S-1}+1}^{R_{S-1}}
    \sigma_{r_{S-1}}^{2}(\mat{w}_{r_1 \cdots r_{S-2}}^{(S-1)})
    \label{eq:appendix-reconstruction-error-unrolled}
\end{multline}
where $\vec{r} = (\widehat{R}_1, \dots, \widehat{R}_{S-1})$ contains the rank values, $\mat{D}_s = \vec{d}^{(s)}$ contains the Kronecker factor shapes and is referred to as the $\mat{D} \vec{r}$-SeKron approximation error (note that the dependency of intermediate matrices $\mat{w}_{r_1 \cdots r_{k-1}}^{(k)}$ on Kronecker factor shapes $\mat{D}$ is implied).
Selecting $\widehat{R}_i = R_i \: \forall i$ in \eqref{eq:appendix-reconstruction-error-unrolled} results in zero reconstruction error.
\end{proof}

\sekrongenerality*
\begin{proof}
The SeKron decomposition of tensor $\tensor{w}$ is given by
\begin{equation}
    \tensor{w}_{i_1 \cdots i_N} = 
    \sum_{r_1, \dots, r_S = 1}^{R_1, \cdots, R_S} 
    \tensor{a}_{r_1 j_1^{(1)} \cdots j_N^{(1)}}^{(1)} \cdots
    \tensor{a}_{r_1 \cdots r_{S-1} j_1^{(S)} \cdots j_N^{(S)}}^{(S)}
    \label{eq:appendix-ksd-summations-outside-scalar-form-full-rank}
\end{equation}
where $\tensor{a}^{(k)} \in \R^{a_1^{(k)} \times \cdots a_N^{(k)}}$ and \begin{equation}
j_n^{(k)} =
    \begin{cases}
        i_n - \sum_{t=1}^{k-2} j_{n}^{(t)} \prod_{l=t+1}^{S} a_n^{(l)} \text{mod} \, a_n^{(S)} \quad & k=S, \\
    \left\lfloor \frac
    { i_n - \sum_{t=1}^{k-1} j_{n}^{(t)} \prod_{l=t+1}^{S} a_n^{(l)} }
    {\prod_{l=k+1}^{S} a_n^{(l)}}\right\rfloor \quad 
    & \text{otherwise},
    \end{cases}
    \tag{\ref{eq:kron-tensor-sequences-indexing} revisted}
\end{equation}
The CP decomposition of tensor $\tensor{w}$ in scalar form is
\begin{equation}
    \tensor{w}_{i_1 \cdots i_N} = \sum_{r=1}^{R^{(\text{CP})}} \tensor{a}_{r i_1}^{(\text{CP}_1)} \cdots \tensor{a}_{r i_N}^{(\text{CP}_N)}
\end{equation}
where $\tensor{a}^{(\text{CP}_k)} \in \R^{R^{(\text{CP})} \times w_k}$. 
Configuring the SeKron decomposition in \eqref{eq:appendix-ksd-summations-outside-scalar-form-full-rank} such that $S=N; \; R_1=R^{(\text{CP})}; \; R_2, \dots, R_N=1$ and $a_n^{(n)}=w_n$ for $n=1, \dots, N$ leads to the equivalent form
\begin{equation}
    \tensor{w}_{i_1 \cdots i_N} = 
    \sum_{r_1=1}^{R^{(\text{CP})}} 
    \tensor{a}_{r_1 i_1 1 \cdots 1}^{(1)} \cdots
    \tensor{a}_{r_, 1 \cdots 1 i_N}^{(N)}.
    \label{eq:appendix-ksd-summations-outside-scalar-form}
\end{equation}
The Tucker decomposition of tensor $\tensor{w}$ is given by
\begin{equation}
    \tensor{w}_{i_1 \cdots i_N} = 
    \sum_{r_1=1, \dots, r_N}^{R_1^{(\text{T})}, \dots R_N^{(\text{T})}}
    \tensor{g}_{r_1 \cdots r_N} 
    \tensor{a}_{i_1 r_1}^{(\text{T}_1)}
    \cdots
    \tensor{a}_{i_N r_N}^{(\text{T}_N)}
    \label{appendix-tucker-scalar-form}
\end{equation}
where $\tensor{g} \in \R^{R_1^{(\text{T})} \times \cdots \times R_N^{(\text{T})}}$ and $\tensor{a}^{(\text{T}_k)} \in \R^{w_k \times R_k^{(\text{T})}}$. The SeKron decomposition of tensor $\tensor{w}$, with $S=N+1, \; R_n=R_{n}^{(T)}$ and $ a_n^{(n)} = w_n $ for $n=1, \dots, N$ yields
\begin{equation}
    \tensor{w}_{i_1 \cdots i_N} = 
    \sum_{r_1, \dots, r_{N}=1}^{R_1^{(\text{T})}, \cdots, R_N^{(\text{T})}} 
    \tensor{a}_{r_1 i_1 1 \cdots 1}^{(1)} \cdots
    \tensor{a}_{r_1 \cdots r_{N} 1 \cdots 1 i_N}^{(N)}
    \tensor{a}_{r_1 \cdots r_{N} 1 \cdots 1}^{(N+1)},
\end{equation}
which is equivalent to \eqref{appendix-tucker-scalar-form} in the special case where there are nullity constraints on some elements in the Kronecker factors, such that for $k=2,\dots,N$  
\begin{equation}
    \tensor{a}^{(k)}_{r_1 \cdots r_k 1 \cdots 1 i_k 1 \cdots 1} = 0 \quad \text{when} \quad r_j \in \{ x\in\mathbb{N} \;|\; x \leq R_j^{(\text{T})}, \; x \neq R_j^{(\text{T}^*)} \} \quad j=1,\dots,k-1
\end{equation}
for any choice of $R_j^{(\text{T}*)} \in \{ x\in\mathbb{N} \;|\; x \leq R_j^{(\text{T})}\}$.
The Tensor Ring (TR) decomposition of $\tensor{w}$ is given by
\begin{equation}
    \tensor{w}_{i_1 \cdots i_N} = 
    \sum_{r_1=1, \dots, r_N}^{R_1^{(\text{TR})}, \dots R_N^{(\text{TR})}}
    \tensor{a}_{i_1 r_1 r_2}^{(\text{TR}_1)}
    \cdots
    \tensor{a}_{i_N r_N r_{N+1}}^{(\text{TR}_N)}
    \label{appendix-tr-scalar-form}
\end{equation}
where $\tensor{a}^{(\text{TR}_k)} \in \R^{w_k \times R_k^{(\text{TR})} \times R_{k+1}^{(\text{TR})}}$, and $R_{1}^{(\text{TR})} = R_{N+1}^{(\text{TR})}$. As the Tensor Train decomposition can be viewed as a special case of the Tensor Ring decomposition (with $R_{1}^{(\text{TR})} = R_{N+1}^{(\text{TR})} = 1$), it suffices to show that SeKron generalizes Tensor Ring. The SeKron decomposition of tensor $\tensor{w}$, with $S=N+1$; $R_k=R_{k}^{(\text{TR})}$ for $k=1, \dots, N-1$ and $ a_n^{(n+1)}=w_n $ for $n=1, \dots, N$ leads to
\begin{equation}
    \tensor{w}_{i_1 \cdots i_N} = 
    \sum_{r_1, \dots, r_{N}=1}^{R_1^{(\text{TR)}}, \dots, R_N^{(\text{TR})}} 
    \tensor{a}_{r_1 1 \cdots 1}^{(1)} 
    \tensor{a}_{r_1 r_2 i_1 1 \cdots 1}^{(2)}
    \cdots
    \tensor{a}_{r_1 \cdots r_{N+1} 1 \cdots 1 i_N}^{(N+1)},
\end{equation}
which is equivalent to \eqref{appendix-tr-scalar-form} in the special case where some elements in the Kronecker factors are constrained, such that all elements in tensor $\tensor{a}^{(1)}$ are constrained to one and 
\begin{equation}
\tensor{a}^{(k)}_{r_1 \cdots r_k 1 \cdots 1 i_k 1 \cdots 1} = 0 \quad  \forall r_j \in \{x\in \mathbb{N}  \;|\;  x \leq R_j^{(\text{TR})}, \; x \neq R_j^{(\text{TR}^*)} \}
\end{equation}
for
\begin{equation}
    j = \begin{cases}
        1, \dots, k-2 & k=2, \dots, N \\
        2,\dots,k-1 & k=N+1
    \end{cases}
\end{equation}
for any choice of $R_j^{(\text{TR}^*)} \in \{ x\in\mathbb{N} \;|\; x \leq R_j^{(\text{TR})}\}$. 
\end{proof}

\ksdconv*
\begin{proof}
First we bring out the summations in the SeKron representaion of $\tensor{w}$ 
\begin{equation}
    \tensor{w} =  \sum_{r_{1}}^{R_{1}} \tensor{a}_{r_{1}}^{(1)} \otimes
    \sum_{r_{2}}^{R_{2}} \tensor{a}_{r_{1} r_{2}}^{(2)} \otimes\cdots\otimes
    \sum_{r_{S-1}}^{R_{S-1}} \tensor{a}_{r_{1} \cdots r_{S-1} }^{(S-1)}
    \otimes
    \tensor{a}_{r_{1} \cdots r_{S-1} }^{(S)},
    \tag{\ref{eq:ksd-full-rank-tensor-form} revisted}
\end{equation}
such that
\begin{equation}
    \tensor{w} = 
    \sum_{r_1, \dots, r_S = 1}^{R_1, \cdots, R_{S-1}} \tensor{a}_{r_1}^{(1)} 
    \otimes
    \dots
    \otimes 
    \tensor{a}_{r_1 r_2 \cdots r_{S-1}}^{(S)}.
    \label{eq:appendix-ksd-summations-outside-tensor-form}
\end{equation}
Then, using the scalar form definition of sequences of kronecker products in \eqref{eq:kron-tensor-sequences-indexing}
\begin{equation}
j_n^{(k)} =
    \begin{cases}
        i_n - \sum_{t=1}^{k-2} j_{n}^{(t)} \prod_{l=t+1}^{S} a_n^{(l)} \text{mod} \, a_n^{(S)} \quad & k=S, \\
    \left\lfloor \frac
    { i_n - \sum_{t=1}^{k-1} j_{n}^{(t)} \prod_{l=t+1}^{S} a_n^{(l)} }
    {\prod_{l=k+1}^{S} a_n^{(l)}}\right\rfloor \quad 
    & \text{otherwise},
    \end{cases}
    \tag{\ref{eq:kron-tensor-sequences-indexing} revisited}
\end{equation}
allows us to re-write \eqref{eq:appendix-ksd-summations-outside-tensor-form} in scalar form as
\begin{equation}
    \tensor{w}_{i_1 \cdots i_N} = 
    \sum_{r_1 \cdots r_S = 1}^{R_1} 
    \tensor{a}_{r_1 j_1^{(1)} \cdots j_N^{(1)}}^{(1)} \cdots
    \tensor{a}_{r_1 \cdots r_{S-1} j_1^{(S)} \cdots j_N^{(S)}}^{(S)}
    \label{eq:appendix-ksd-summations-outside-scalar-form}
\end{equation}
As the $j_n^{(k)}$ terms decompose $i_n$ into an integer weighted sum, we can recover $i_n$ using 
\begin{equation}
    i_n = f(\vec{j}_n) \triangleq \sum_{k=1}^{S} j_n^{(k)} \prod_{l=k+1}^{S} a_n^{(l)},
\end{equation}
where $\vec{j}_n = (j_n^{(1)}, \dots, j_n^{(S)})$. Thus, we can write 
\begin{equation}
    \tensor{x}_{i_1 + z_1, \cdots i_N + z_N} 
    = 
    \tensor{x}_{f(\vec{j}_1) + z_1, \cdots f(\vec{j}_N) + z_N}.
    \label{eq:re-indexing-tensor-x}
\end{equation}
Finally, combining equations \eqref{eq:appendix-ksd-summations-outside-scalar-form} and \eqref{eq:re-indexing-tensor-x} leads to \eqref{eq:linear-mappings-with-kronecker-sequences-scalar-form}.
\end{proof}

\begin{theorem} (Universal approximation via shallow SeKron networks) Any shallow SeKron factorized neural network $\hat{f}^{(s)}$ with an $L$-Lipschitz activation function $a$, is dense in the class of continuous functions $C(X)$ for any compact subset $X$ of $\R^d$
\end{theorem}
\begin{proof}
Let $\hat{f}$ denote a shallow neural network, and $f \in C(X)$. Then,
\begin{align}
    \left\|f - \hat{f}^{(s)}\right\|_2^2 
    &\triangleq 
    \int_X \Bigg(f(x) - \hat{f}^{(s)}(x) \Bigg)^2 d\mu \\
    \label{eq:appendix-shallow-norm-full-rank}
    &=
    \int_X \Bigg(f(x) - \hat{f}(x) \Bigg)^2 d\mu \\
    \label{eq:appendix-shallow-norm-low-rank}
    &+ 
    \int_X \Bigg(\hat{f}(x) - \hat{f}^{(s)}(x) \Bigg)^2 d\mu \\
    &+
    2\int_X \Bigg(f(x) - \hat{f}(x) \Bigg) \Bigg(\hat{f}(x) - \hat{f}^{(s)}(x) \Bigg) d\mu 
\end{align}
According to Hornik (1991), \eqref{eq:appendix-shallow-norm-full-rank} is dense in $C(X)$; therefore, it suffices to show that \eqref{eq:appendix-shallow-norm-low-rank} is bounded as well.
\begin{align}
    \int_X \Bigg(\hat{f}(x) - \hat{f}^{(s)}(x) \Bigg)^2 d\mu 
    &=
    \int_X \Bigg( \vec{w}^{\top}\vec{a}(\mat{w}\vec{x}) - \vec{w}^{\top}\vec{a}(\mat{w}^{(s)}\vec{x}) \Bigg)^2 d\mu \\
    &\leq
    L\big\|\vec{w}\big\|_2^{2}\big \|X\big\|_2^{2} \, \varepsilon_{\text{SeKron}}(\mat{W}, \vec{r}, \mat{D}) 
\end{align}
where $\varepsilon$ denotes the $\mat{D} \vec{r}$-SeKron approximation error as in \eqref{eq:appendix-reconstruction-error-unrolled}, with matrix $\mat{D}$ and vector $\vec{r}$ describing the shapes of the Kronecker factors the ranks used in the SeKron decomposition of $\mat{w}$, respectively.
\end{proof}